\documentclass{article} 
\usepackage{arxiv}  
\usepackage{times}  
\usepackage{helvet}  
\usepackage{courier}  
\usepackage[hyphens]{url}  
\usepackage{graphicx} 
\usepackage{hyperref}
\usepackage{caption} 
\usepackage{algorithm}
\usepackage{algorithmic}
\usepackage{subfigure}
\usepackage{booktabs}

\usepackage{amsmath}
\usepackage{amssymb}
\usepackage{mathtools}
\usepackage{amsthm}

\theoremstyle{plain}
\newtheorem{theorem}{Theorem}

\newtheorem{lemma}{Lemma}

\theoremstyle{definition}

\theoremstyle{remark}

\usepackage{newfloat}
\usepackage{listings}
\DeclareCaptionStyle{ruled}{labelfont=normalfont,labelsep=colon,strut=off} 
\floatstyle{ruled}
\newfloat{listing}{tb}{lst}{}
\floatname{listing}{Listing}
\title{Episodic Return Decomposition by Difference of Implicitly Assigned Sub-Trajectory Reward}
\author{%
    Haoxin Lin\textsuperscript{\rm 1,\rm 2,\rm 3},~~
    Hongqiu Wu\textsuperscript{\rm 1,\rm 2},~~
    Jiaji Zhang\textsuperscript{\rm 1,\rm 2},~~
    Yihao Sun\textsuperscript{\rm 1,\rm 2},~~
    Junyin Ye\textsuperscript{\rm 1,\rm 2,\rm 3},~~
    Yang Yu\textsuperscript{\rm 1,\rm 2,\rm 3,\rm 4}\thanks{Corresponding Author} \\   
    \textsuperscript{\rm 1}National Key Laboratory for Novel Software Technology, Nanjing University, Nanjing, China\\
    \textsuperscript{\rm 2}School of Artificial Intelligence, Nanjing University, Nanjing, China\\
    \textsuperscript{\rm 3}Polixir Technologies, Nanjing, China\\
    \textsuperscript{\rm 4}Peng Cheng Laboratory, Shenzhen, 518055, China\\
    \{linhx, wuhq, zhangjj, sunyh, yejy\}@lamda.nju.edu.cn, yuy@nju.edu.cn
}
\date{}

\begin{document}

\maketitle

\begin{abstract}
Real-world decision-making problems are usually accompanied by delayed rewards, which affects the sample efficiency of Reinforcement Learning, especially in the extremely delayed case where the only feedback is the episodic reward obtained at the end of an episode. Episodic return decomposition is a promising way to deal with the episodic-reward setting. Several corresponding algorithms have shown remarkable effectiveness of the learned step-wise proxy rewards from return decomposition. However, these existing methods lack either attribution or representation capacity, leading to inefficient decomposition in the case of long-term episodes. In this paper, we propose a novel episodic return decomposition method called Diaster (\textbf{D}ifference of \textbf{i}mplicitly \textbf{a}ssigned \textbf{s}ub-\textbf{t}rajectory \textbf{re}ward). Diaster decomposes any episodic reward into credits of two divided sub-trajectories at any cut point, and the step-wise proxy rewards come from differences in expectation. We theoretically and empirically verify that the decomposed proxy reward function can guide the policy to be nearly optimal. Experimental results show that our method outperforms previous state-of-the-art methods in terms of both sample efficiency and performance. The code is available at https://github.com/HxLyn3/Diaster.
\end{abstract}

\section{Introduction}
\label{intro}
Reinforcement Learning (RL) has made empirical successes in several simulated domains and attracted close attention in recent years. Sparse and delayed reward, usually facing a range of real-world problems such as industrial control \cite{industry}, molecular design \cite{molecule}, and recommendation systems \cite{recom}, is one of the critical challenges hindering the application of RL in reality. The trouble with the delayed reward function is that it affects RL's sample efficiency when using it to estimate the value function. Specifically, delayed rewards introduce numerous noneffective updates and fitting capacity loss \cite{infer} into Temporal-Difference (TD) Learning and result in a high variance in Monte-Carlo (MC) Learning \cite{rudder}. The problem is more severe in the highly delayed case that the feedback appears only at the end of the episode.

It is essential to design a proxy Markovian reward function as the substitute for the original delayed one since current standard RL algorithms prefer instant feedback at each environmental step. The apparent solution is designing the required proxy reward function by handcraft, but it depends on domain knowledge and is not general for all tasks. Reward shaping \cite{Mataric94, bicycleshaping, pbrs}, one kind of automatical reward design method, is widely observed that it can accelerate RL. However, the reshaped rewards are coupled with the original rewards and are ineffective in environments where only an episodic reward is fed back. Recent work \cite{rudder, ircr, rd, rrd} proposes to decompose the episodic reward into a string of step-wise rewards through the trajectory, which seems promising to deal with the long-term delay.

\begin{figure*}[pt!]
\centering
\subfigure[Attribute a return to the full trajectory.\quad]{
    \centering
    \includegraphics[height=0.195\linewidth]{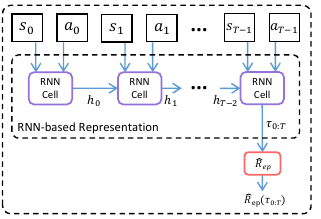}
    \label{intro_fig1}
}%
\subfigure[Attribute a return to each step.\quad]{
    \centering
    \includegraphics[height=0.195\linewidth]{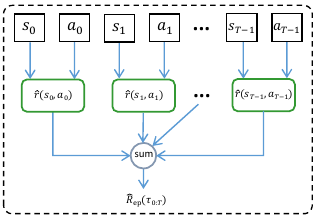}
    \label{intro_fig2}
}%
\subfigure[Attribute a return to any two cut sub-trajectories.]{
    \centering
    \includegraphics[height=0.195\linewidth]{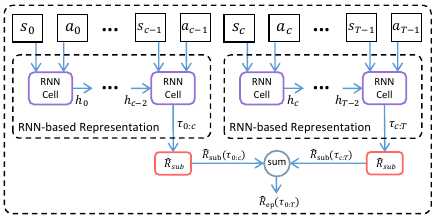}
    \label{intro_fig3}
}%
\caption{Comparison of different episodic return decomposition. (a) introduces a temporal structure to represent the episodic return function $\hat{R}_{\text{ep}}(\tau_{0:T})$. The step-wise proxy rewards can be assigned via contribution analysis from a backward view. (b) directly decomposes the episodic return into the step-wise proxy rewards $\{\hat{r}(s_i,a_i)\}_{i=0}^T$. (c) decomposes the episodic return into two sub-trajectory rewards, $\hat{R}_{\text{sub}}(\tau_{0:c})$ and $\hat{R}_{\text{sub}}(\tau_{c:T})$, for any cut point $c$.}
\label{intro_fig}
\end{figure*}

Previous return decomposition methods can be sorted into two classes, as shown in Figure \ref{intro_fig1} and \ref{intro_fig2}. One pays attention to precisely predicting the long-term episodic return, thinking little of reward redistribution. The representative work is RUDDER \cite{rudder}, which utilizes an LSTM \cite{lstm} network to make a return prediction and assigns rewards to each pair of state-action via contribution analysis from a backward view. Another focuses on step-wise return decomposition without considering temporal structural representations. A little work \cite{rd, podr, rrd} aimed at the least-squares-based return decomposition objective falls into this class. Nevertheless, the above work suffers from a lack of either attribution or representation capacity, leading to inefficient episodic return decomposition in the case of long-term episodes. 

In this paper, we propose to cut an entire trajectory at any time step and decompose the episodic return into two sub-trajectory rewards, as demonstrated in Figure \ref{intro_fig3}. The step-wise reward can be straightforwardly obtained by differencing the implicitly assigned sub-trajectory reward in expectation. Based on this key idea, we formulate a simple yet effective episodic return decomposition method called Diaster (\textbf{D}ifference of \textbf{i}mplicit \textbf{a}ssigned \textbf{s}ub-\textbf{t}rajectory \textbf{re}ward), which can be embedded in any model-free RL algorithm. This new class of return decomposition not only introduces temporal structural representations into the episodic reward function but also properly attributes the return to different parts of a given trajectory. Hence, a practical proxy reward function can be learned to guide the policy to be optimal.

In general, our contributions are summarized as follows:
\begin{itemize}
    \item We elaborate on the formulation of our new episodic return decomposition method, \emph{i.e.}, Diaster, and present its practical neuron-based implementation.
    \item We theoretically verify that the step-wise proxy reward function learned through Diaster is return-equivalent to the original MDP in expectation and capable of guiding the policy to be nearly optimal.
    \item We empirically show that Diaster can provide effective proxy rewards for RL algorithms and outperform previous state-of-the-art return decomposition methods in terms of both sample efficiency and performance.
    \item We conduct an ablation study to show that setting the number of cut points for Diaster can achieve a trade-off between long-term representation and attribution.
\end{itemize}

\section{Related Work}
\label{related}
\subsection{Delayed Reward}
The setting of delayed rewards \cite{delay1, delay2, delay3} usually accompanies real-world RL problems, resulting in a high bias in TD Learning and high variance in MC Learning \cite{rudder}. This paper considers the extremely delayed case in which only one episodic feedback can be obtained at the end of an episode. The problem is how to automatically find a proxy dense reward function to encourage an efficient value estimation in any task with the episodic-reward setting.

\subsection{Reward Shaping}
Reward shaping is one of the popular methods of replacing the original rewards with more frequent feedback to accelerate RL agents' learning. Early work \cite{robotshaping, Mataric94, bicycleshaping} focuses on designing the shaping reward function and shows an improved learning efficiency. However, these methods need more consideration for the consistency of the optimal policy. Potential-based reward shaping (PBRS) \cite{pbrs} proposes to guarantee the policy invariance property by restricting the shaping function to a difference of potential functions defined over states. Several variants \cite{pba, dpbrs, dpba} of PBRS introduce more information into the potential function to provide additional expressive power for proxy rewards and improve the sample efficiency. These potential-based methods do not correspond to an optimal return decomposition as their proxy rewards are additively coupled with the original rewards, according to \cite{rudder}. 

\subsection{Return Decomposition}
Episodic return decomposition is promising to tackle the rewards with an extreme delay. RUDDER \cite{rudder} uses an LSTM network \cite{lstm} to predict the return of an episode and decompose the predicted return into contributions of each pair of state-action in the sequence. \cite{rd} directly optimizes the least-squares-based decomposition objective, \emph{i.e.}, the expected square error between the sum of the state-action pairs' proxy rewards and the corresponding trajectory-wise reward. This direct return decomposition method is extended by \cite{podr} to decompose the episodic return into the rewards of all time intervals. IRCR \cite{ircr} considers a uniform reward redistribution of the return to each constituent state-action pair, which is an intuitive solution of the optimal reward problem (ORP) \cite{wheredr, imrl}. To bridge direct return decomposition and IRCR, RRD \cite{rrd} lets the proxy rewards from a random subsequence of the trajectory to fit the episodic return, achieving a trade-off between the decomposition complexity and the estimation accuracy.

\subsection{Temporal Credit Assignment}
Long-term temporal credit assignment (TCA) \cite{TCA}, attributing agents' actions to future outcomes that may be observed after a long time interval, is a long-lasting problem in RL. General TCA frameworks \cite{adaptlambda, pairw} are proposed to extend the $\lambda$-return \cite{lambda}, which is one of the earliest heuristic mechanisms for TCA. Hindsight credit assignment (HCA) \cite{hca}
assigns credit to past actions by reasoning in the backward view, opening up a new family of algorithms. Temporal value transport (TVT) \cite{tvt} augments the capabilities of memory-based agents and uses recall of specific memories to credit actions from the past, achieving an efficient credit transport. 

\section{Preliminaries}
\label{prelim}
A finite-horizon Markov Decision Process (MDP) with an episodic-reward setting is characterized by a tuple $\langle\mathcal{S}, \mathcal{A}, P, \rho_0, r, R_{\text{ep}}, T\rangle$. The episodic horizon is limited by the scalar $T$. Each episode starts with an environmental state $s_0\in\mathcal{S}$ sampled from the initial state distribution $\rho_0$. At each time step $t\in\{0,\ldots,T-1\}$, after agent observing the state $s_t$ and performing a $a_t\in\mathcal{A}$, the environment transitions to the next state $s_{t+1}\in\mathcal{S}$ according to the unknown transition function $P(s_{t+1}|s_t,a_t)$. The step-wise reward $r(s_t,a_t)$ is unobservable to the agent, while the only feedback called episodic reward $R_{\text{ep}}(\tau_{0:T})$ is accessible at the end of the trajectory $\tau_{0:T}=(s_0,a_0,s_1,a_1,...,s_{T-1},a_{T-1})$. In this paper, we consider an assumption that the episodic reward comes from the sum of the step-wise rewards, \emph{i.e.}, $R_{\text{ep}}(\tau_{0:T})=\sum_{t=0}^{T-1}r(s_t,a_t)$, which is satisfied in most cases. The goal under this setting is to find the optimal policy that maximizes the expectation of episodic return (\emph{i.e.}, episodic reward): $\mathbb{E}_{\rho^\pi}\left[R_{\text{ep}}(\tau_{0:T})\right]$, where $\rho^\pi$ induced by the dynamics function $P(s_{t+1}|s_t,a_t)$ and the policy $\pi(a_t|s_t)$ is the on-policy distribution over states.

The state-action value function defined as
\begin{equation}
Q^\pi(s_t,a_t)=\mathbb{E}_{P,\pi}\left[\left.\sum_{i=0}^{T-t-1}r(s_{t+i},a_{t+i})\right|s_t,a_t\right]
\end{equation}
cannot be estimated since the immediate reward isn't fed back from the environment. Directly using the delayed episodic reward to learn the state-action value function leads to the optimal policy consistently, as proven by \cite{rudder}. However, it isn't an appropriate choice as it introduces high bias in TD Learning and high variance in MC Learning. The promising way to solve any MDP with an episodic-reward setting is to find a proxy dense reward function $\hat{r}(s,a)$ that nearly satisfies the sum-from decomposition,
\begin{equation}
\label{rd}
R_{\text{ep}}(\tau_{0:T})\approx\hat{R}_{\text{ep}}(\tau_{0:T})=\sum_{t=0}^{T-1}\hat{r}(s_t,a_t),
\end{equation}
which is considered in several previous work \cite{rudder, rd, podr, rrd, srltca}. 

\section{Diaster: Difference of Implicitly Assigned Sub-Trajectory Reward}
\label{diaster}
In this section, we will first elaborate on our return decomposition method Diaster (\textbf{D}ifference of \textbf{i}mplicit \textbf{a}ssigned \textbf{s}ub-\textbf{t}rajectory \textbf{re}ward), and next verify the effectiveness of this method theoretically. Then we will present a practical neuron-based implementation of Diaster.

\subsection{Description of Diaster}
Rather than directly decomposing the episodic reward into the step-wise rewards, we consider finding a proxy sub-trajectory reward function $\hat{R}_{sub}$ that satisfies
\begin{gather}
\hat{R}_{\text{sub}}(\varnothing) =0,  \\
\hat{R}_{\text{sub}}(\tau_{0:T}) \approx R_{\text{ep}}(\tau_{0:T}),\\
\hat{R}_{\text{sub}}(\tau_{0:c}) + \hat{R}_{\text{sub}}(\tau_{c:T}) \approx R_{\text{ep}}(\tau_{0:T}), \forall c\in\{1,\ldots,T-1\}, \label{strd}
\end{gather}
for any episode $\tau_{0:T}$, where $\tau_{0:c}$ is the former $c$-step sub-trajectory, $\tau_{c:T}$ is the latter $\{T-c\}$-step sub-trajectory, and $\hat{R}_{\text{sub}}(\varnothing)$ is the supplementary definition of the empty sub-trajectory $\tau_{0:0}=\varnothing$. Compared to the step-wide decomposition \eqref{rd}, this formulation of decomposition can ease the difficulty of credit assignment for several reasons:
\begin{itemize}
    \item Only two components need to be assigned the reward after cutting an episode at any cut-point $c$.  
    \item $T$ cut-points to cut an episode provides more relations between the proxy reward function and the episodic reward function, increasing the training samples for episodic return decomposition.
    \item The step-wise decomposition \eqref{rd} can be regarded as a special case of the subtrajectory-wise decomposition \eqref{strd} that lets $\hat{R}_{\text{sub}}(\tau_{0:c})=\sum_{t=0}^{c-1}\hat{r}(s_t,a_t)$ and $\hat{R}_{\text{sub}}(\tau_{c:T})=\sum_{t=c}^{T-1}\hat{r}(s_t,a_t)$. However, a sub-trajectory reward does not have to come from the sum of its state-action pairs' rewards in practice, while it can be any other function of the sub-trajectory. The general formulation can relax the assumption that decides how the episodic reward comes.
    \item This form of return decomposition introduces some temporal structural representations into the episodic reward function, encouraging the utilization of RNNs \cite{rnn} to take advantage of the sequential information.
\end{itemize} 

Given any episode $\tau_{0:T}$ sampled by the policy $\pi$, the proxy reward connected with its preceding sub-trajectory is obtained by difference as
\begin{gather}
    \hat{R}_d(s_0, a_0) = \hat{R}_{\text{sub}}(\tau_{0:1}) - \hat{R}_{\text{sub}}(\varnothing),\\
    \hat{R}_d(\tau_{0:t}, s_t, a_t) = \hat{R}_{\text{sub}}(\tau_{0:t+1}) - \hat{R}_{\text{sub}}(\tau_{0:t}),
\end{gather}
for any $(s_t, a_t)$ belonging to $\tau_{0:T}$.
The method of difference is also considered by \cite{rudder} to guarantee the equivalent episodic return since
\begin{equation}
    \sum_{t=0}^{T-1}\hat{R}_d(\tau_{0:t}, s_t, a_t)=\hat{R}_{\text{sub}}(\tau_{0:T}) \approx R_{\text{ep}}(\tau_{0:T}).
\end{equation}
However, the above reward function $\hat{R}_d$ is still delayed and non-Markovian. Thus we define the step-wise proxy reward function as
\begin{equation}
\label{stepr}
    \hat{r}^\pi(s_t,a_t)=\mathbb{E}_{\tau_{0:t}\sim\Gamma^\pi_t(\cdot|s_t)}\left[\hat{R}_d(\tau_{0:t}, s_t, a_t)\right],
\end{equation}
where
\begin{equation}
    \Gamma^\pi_t(\tau_{0:t}|s_t)=\frac{\mathcal{T}_t^\pi(\tau_{0:t})P(s_t|s_{t-1},a_{t-1})}{\rho^\pi(s_t)}
\end{equation}
is the distribution conditioned on the state $s_t$ over the $t$-step sub-trajectory space, while $\mathcal{T}_t^\pi(\tau_{0:t})$ is the unconditioned distribution that comes from
\begin{equation}
    \mathcal{T}_t^\pi(\tau_{0:t})=\rho_0(s_0)\pi(a_0|s_0)\prod_{i=1}^{t-1}P(s_i|s_{i-1},a_{i-1})\pi(a_i|s_i).
\end{equation}

\subsection{Analysis of Diaster}
Our analysis in this subsection focuses on determining whether the proxy reward function and the corresponding proxy state-action value function are effective for policy optimization. We start by presenting a simple theorem, offering some insights about the relationship between the step-wise reward $\hat{r}^\pi$ and the sub-trajectory reward $\hat{R}_{\text{sub}}$.
\begin{theorem}
\label{theorem1}
Given any policy $\pi$, the following equation holds for any subsequence length $h\in\{2,\ldots,T\}$ and time step $t\in\{0,\ldots,T-h\}$, that is, 
\begin{equation}
    \mathbb{E}_{\rho^\pi}\left[\sum_{i=t}^{t+h-1}\hat{r}^\pi(s_{i},a_{i})\right]=\mathbb{E}_{\rho^\pi}\left[\hat{R}_{\text{sub}}(\tau_{0:t+h})-\hat{R}_{\text{sub}}(\tau_{0:t})\right].
\end{equation}
\end{theorem}
\begin{proof}
See Appendix A.1.
\end{proof}
This theorem states that the sum of any consecutive $h$-step sequence of $\hat{r}^\pi$ is the difference of two sub-trajectory rewards that differ by $h$ steps, in expectation. In particular, letting $t=0$ and $h=T$, it becomes
\begin{equation}
\mathbb{E}_{\rho^\pi}\left[\sum_{t=0}^{T-1}\hat{r}^\pi(s_{t},a_{t})\right]\approx\mathbb{E}_{\rho^\pi}\left[R_{\text{ep}}(\tau_{0:T})\right],
\end{equation}
which reveals that $\hat{r}^\pi$ is one of the return-equivalent proxy rewards of the original MDP. Therefore, the goal of policy optimization does not change while using $\hat{r}^\pi$ as the immediate guidance.

Next we aim to show that using the proxy reward $\hat{r}^{\pi}(s_t,a_t)$ to learn the proxy state-action value function defined as
\begin{equation}
    \hat{Q}^\pi(s_t,a_t)=\mathbb{E}_{P,\pi}\left[\left.\sum_{i=0}^{T-t-1}\hat{r}^\pi(s_{t+i},a_{t+i})\right|s_t,a_t\right],
\end{equation}
nearly leads to the optimal policy in the original MDP. The necessary condition that $\hat{Q}^\pi(s_t,a_t)$ needs to satisfy for the optimal policy consistency is provided by Lemma \ref{theorem2}.
\begin{lemma}
\label{theorem2}
For any finite MDP with a state-action value function $Q^\pi(s,a)$ defined based on the unobservable step-wise reward function $r(s,a)$, any proxy state-action value function $\hat{Q}^\pi(s,a)$ that satisfies
\begin{equation}
\label{cond}
    \hat{Q}^\pi(s,a)=Q^\pi(s,a)+\delta(s),
\end{equation}
where $\delta(s)$ is any function defined over the state space, will lead to the same optimal policy as $Q^\pi$, whether using value-based or policy gradient methods.
\end{lemma}
\begin{proof}
See Appendix A.2.
\end{proof}
This lemma indicates that once the difference of $\hat{Q}^\pi(s,a)$ and $Q^\pi(s,a)$ is only related to the state $s$, the original optimal policy can be guaranteed even though using $\hat{Q}^\pi(s,a)$ to optimize the policy. In order to verify that $\hat{Q}^\pi(s,a)$ can satisfy the condition \eqref{cond}, we further extend Theorem \ref{theorem1} to the conditional form.
\begin{theorem}
\label{theorem3}
Given any policy $\pi$, the following equation holds for any $s_t\in\mathcal{S}$, $a_t\in\mathcal{A}$ at any time step $t\in\{0,\ldots,T-1\}$, that is,
\begin{equation}
\mathbb{E}_{P,\pi}\left[\left.\sum_{i=1}^{T-t-1}\hat{r}^\pi(s_{t+i},a_{t+i})\right|s_t,a_t\right]
=\mathbb{E}_{P,\pi}\left[\left.\hat{R}_\text{sub}(\tau_{0:T})-\hat{R}_\text{sub}(\tau_{0:t+1})\right|s_t,a_t\right].
\end{equation}
\end{theorem}
\begin{proof}
See Appendix A.3.
\end{proof}
By applying Theorem \ref{theorem3}, it is observed that
\begin{equation}
\begin{aligned}
\hat{Q}^\pi(s_t,a_t)
=&\hat{r}^\pi(s_t,a_t)+\mathbb{E}_{P,\pi}\left[\left.\hat{R}_\text{sub}(\tau_{0:T})-\hat{R}_\text{sub}(\tau_{0:t+1})\right|s_t,a_t\right]\\
\approx&\mathbb{E}_{P,\pi}\left[\left.\sum_{i=0}^{T-1}r(s_i,a_i)-\hat{R}_\text{sub}(\tau_{0:t})\right|s_t,a_t\right]\\
\approx&Q^\pi(s_t,a_t)+\mathbb{E}_{P,\pi}\left[\left.\sum_{i=0}^{t-1}r(s_i,a_i)-\hat{R}_\text{sub}(\tau_{0:t})\right|s_t\right],
\end{aligned}
\end{equation}
which conforms the form of the condition \eqref{cond} with $\delta(s_t)=\mathbb{E}_{P,\pi}\left[\left.\sum_{i=0}^{t-1}r(s_i,a_i)-\hat{R}_\text{sub}(\tau_{0:t})\right|s_t\right]$. Hence $\hat{Q}^\pi$ is an effective substitute of $Q^\pi$ to optimize the policy, in the case of original reward $r$ being inaccessible.

\begin{algorithm}[t]
\caption{Diaster}
\label{diaster_algo}
\begin{algorithmic}[1]
\STATE {\bfseries Input:} Neural parameters $\psi$, $\phi$, replay buffer $\mathcal{D}$, learning rate $\lambda_{\hat{R}}$, $\lambda_{\hat{r}}$, batch size $B$, and RL algorithm RL-ALGO containing the initial policy $\pi_\theta$.
\FOR{$N$ episodes}
\STATE Collect a trajectory $\tau_{0:T}$ using the policy $\pi_\theta$.
\STATE Store $\tau_{0:T}$ and its feedback $R_{\text{ep}}(\tau_{0:T})$ in $\mathcal{D}$.
\FOR{$M$ batches}
\STATE Sample $B$ trajectories $\{\tau_{0:T}^i\}_{i=0}^{B-1}$ from $\mathcal{D}$.
\STATE Sample $B$ scalars $\{c^i\}_{i=0}^{B-1}$ from $\{0,\ldots,T-1\}$.
\STATE $\psi\leftarrow\psi-\nabla_\psi\frac{1}{B}\sum_{i=0}^{B-1}l_{\hat{R}}(\psi,\tau_{0:T}^i,c^i)$ by Eq.\eqref{sub_loss}.
\STATE $\phi\leftarrow\phi-\nabla_\phi\frac{1}{B}\sum_{i=0}^{B-1}l_{\hat{r}}(\phi,\tau_{0:c^i+1}^i)$ by Eq.\eqref{step_loss}.
\STATE Sample $B$ transitions $\{(s_i,a_i,s_{i+1})\}_{i=0}^{B-1}$ from $\mathcal{D}$.
\STATE Use RL-ALGO to optimize the policy $\pi_\theta$ with $\{(s_i,a_i,\hat{r}^\phi(s_i,a_i),s_{i+1})\}_{i=0}^{B-1}$.
\ENDFOR
\ENDFOR
\end{algorithmic}
\end{algorithm}

\subsection{Practical Implementation of Diaster}
The primary problem is how to obtain an implicitly assigned sub-trajectory reward that satisfies \eqref{strd}. The error of the approximation \eqref{strd} directly decides the optimality of the policy learned from the proxy rewards. Considering to extract the temporal structures contained in sub-trajectories, we use an RNN with a GRU \cite{gru} cell, parameterized by $\psi$, to represent the sub-trajectory reward function $\hat{R}_{\text{sub}}^\psi$. We train $\hat{R}_{\text{sub}}^\psi$ by minimizing the expected loss of episodic return decomposition:
\begin{equation}
\mathcal{L}_{\hat{R}}(\psi)=\mathop{\mathbb{E}}_{\tau_{0:T}\in\mathcal{D}}\left[
\mathop{\mathbb{E}}\limits_{c\sim U(\{0,\ldots,T-1\})}\left[l_{\hat{R}}(\psi,\tau_{0:T},c)\right]\right],
\end{equation}
with the replay buffer $\mathcal{D}$, and
\begin{equation}
\label{sub_loss}
l_{\hat{R}}(\psi,\tau_{0:T},c)=\left(\hat{R}_{\text{sub}}^\psi(\tau_{0:c})+\hat{R}_{\text{sub}}^\psi(\tau_{c:T})-R_{\text{ep}}(\tau_{0:T})\right)^2.
\end{equation}
The step-wise proxy reward function $\hat{r}^\phi$, with neural parameters $\phi$, is trained to predict the difference of sub-trajectory rewards. The expected loss is
\begin{equation}
\mathcal{L}_{\hat{r}}(\phi)=\mathop{\mathbb{E}}_{\tau_{0:T}\in\mathcal{D}}\left[
\mathop{\mathbb{E}}\limits_{t\sim U(\{0,\ldots,T-1\})}\left[l_{\hat{r}}(\phi,\tau_{0:t+1})\right]\right],
\end{equation}
where
\begin{equation}
\label{step_loss}
l_{\hat{r}}(\phi,\tau_{0:t+1})=\left(\hat{r}^\phi(s_t,a_t)-\left(\hat{R}_{\text{sub}}^\psi(\tau_{0:t+1})-\hat{R}_{\text{sub}}^\psi(\tau_{0:t})\right)\right)^2
\end{equation}
is the prediction loss for any given sub-trajectory $\tau_{0:t+1}=(s_0,a_0,\ldots,s_t,a_t)$. 

The complete algorithm of Diaster is described in Algorithm \ref{diaster_algo}, where RL-ALGO can be any off-policy RL algorithm. It is efficient to apply Diaster in tasks with an episodic-reward setting due to its easy implementation and few hyper-parameters requiring tuning.

\section{Experiments}
\label{exp}
In this section, we focus on four primary questions: 1) What is the step-wise proxy reward function learned through Diaster like? 2) How well does Diaster perform on RL benchmark tasks, in contrast to previous state-of-the-art episodic return decomposition methods? 3) What will happen if decomposing the episodic return into implicit rewards of more than two sub-trajectories? 4) Is it necessary to learn the step-wise proxy reward function?

\subsection{Experimental Setting}
We conduct a series of experiments on MuJoCo \cite{mujoco} and PointMaze \cite{pointmaze} benchmarks with the same episodic-reward setting as \cite{ircr, rrd}. Specifically, we modify the original environment to prevent the agent from accessing the instant reward. Instead, zero signals will be received by the agent at non-terminal states. The only feedback is the episodic reward $R_{\text{ep}}(\tau_{0:T})$ returned at the end of an episode. The episode is set up with a maximum trajectory horizon $T=1000$. Other environmental parameters are kept the same as the default.

During the following experiments, our Diaster is implemented based on SAC \cite{sac}, one of the state-of-the-art model-free RL algorithms in continuous control tasks. The algorithmic hyper-parameters are given in Appendix B.

\begin{figure*}[pt!]
    \centering
    \includegraphics[height=0.23\linewidth]{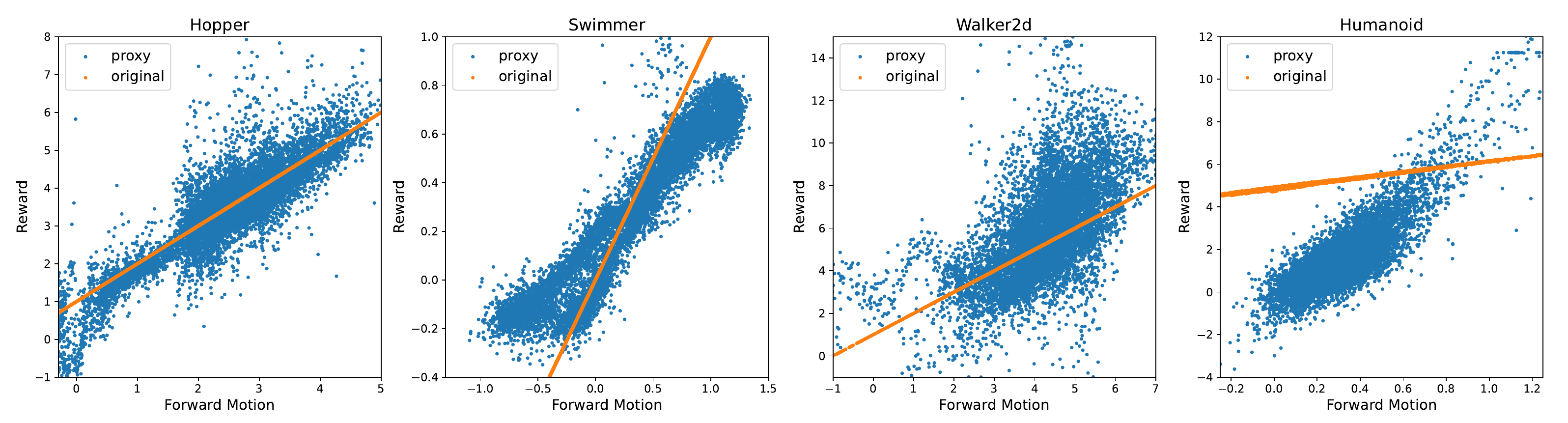}
    \caption{Proxy Reward Visualization. For each sampled pair of state-action $(s, a)$, we show the forward motion of the agent along with the learned proxy reward $\hat{r}^\phi(s,a)$ (blue point) and the original unobservable reward $r(s,a)$ (orange point).}
    \label{reward_exp}
\end{figure*}

\subsection{Proxy Reward Visualization}

We aim to visualize the step-wise proxy rewards learned through our Diaster on four MuJoCo tasks with the episodic-reward setting, including Hopper, Swimmer, Walker2d, and Humanoid. The principal goal of these environments is to let the robot move in the forward direction as fast as possible in the limited time horizon. After that, the episodic reward function mainly connects with the agent's forward motion. To coincide with the episodic goal, the proxy reward function has to be related to the forward motion at any time step. Therefore, revealing the relationship between the learned proxy rewards and the agent's forward motion can indicate the effectiveness of return decomposition. We sample several state-action pairs in the environment and show their proxy rewards along with the step-wise displacements of forward motion. What's more, we show the original environmental rewards for reference. The results are demonstrated in Figure \ref{reward_exp}.

We observe that the point with more distance of forward motion tends to have a greater proxy reward. The nearly monotonic tendency shows how the learned proxy rewards depend on the forward motion and indicates that Diaster can discover the latent attribution of the episodic reward function and make an effective decomposition. Intuitively, using such a proxy reward function can encourage the agent to step forward, consistent with the original goal of MuJoCo tasks.

\subsection{Performance Comparison}
We evaluate our Diaster on seven MuJoCo tasks with the episodic-reward setting, including InvertedPendulum, Hopper, Swimmer, Walker2d, Ant, Humanoid, and HumanoidStandup. Three existing episodic return decomposition methods are selected for comparison:
\begin{itemize}
    \item \textbf{RUDDER} \cite{rudder} trains an LSTM network to predict the episodic reward at every time step. The step-wise proxy reward is assigned by the difference between return predictions of two consecutive states.
    \item \textbf{IRCR} \cite{ircr} considers setting the proxy reward as the normalized value of the corresponding episodic return instead of learning any parametric step-wise reward function.
    \item \textbf{RRD} \cite{rrd} carries out return decomposition on randomly sampled short subsequences of trajectories, which is adaptive to long-horizon tasks and achieves state-of-the-art performance.
\end{itemize}

\begin{figure*}[t!]
    \centering
    \includegraphics[width=0.8\linewidth]{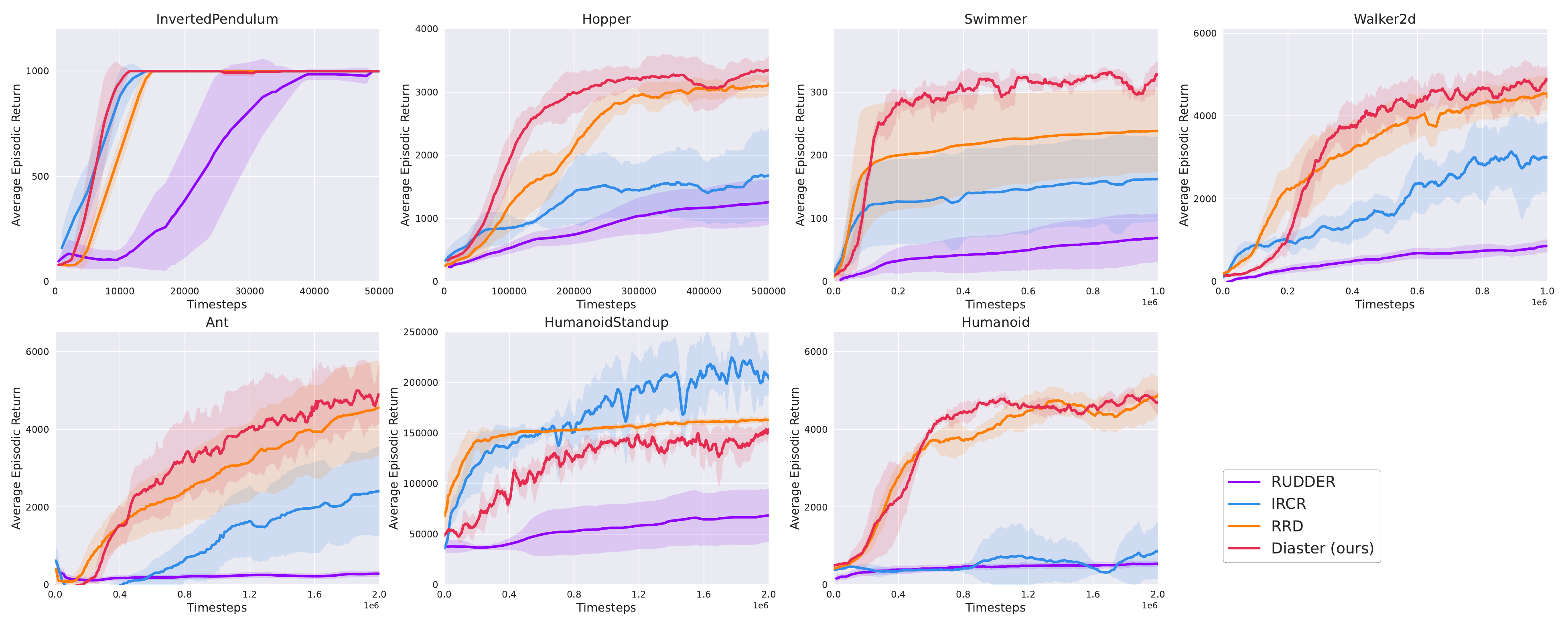}
    \caption{Learning curves of our Diaster (red) and other three baselines on MuJoCo continuous control tasks. The solid lines indicate the mean and shaded areas indicate the standard error over five different random seeds. Each evaluation, taken every 1,000 environmental steps, calculates the average return over ten episodes.}
    \label{performance}
\end{figure*}

Figure \ref{performance} shows the learning curves of our Diaster (red) and the other three baselines. By contrast, RUDDER is no match for Diaster in all seven tasks. As for RRD and IRCR, it can be observed that Diaster shows a better sample efficiency than these two algorithms in most of the environments. Only in the HumanoidStandup task, Diaster has a slight disadvantage compared to RRD and IRCR. Although unable to demonstrate an overwhelming superiority, these results show that Diaster has both high sample efficiency and competitive performance on MuJoCo benchmarks.

\subsection{On Number of Cut Points}
We propose to cut an entire trajectory at any cut point and decompose the episodic return into two sub-trajectory rewards in the above. A natural question is what will happen if decomposing the episodic return into implicit rewards of more than two sub-trajectories. In this subsection, we conduct an ablation study to show how Diaster performs while changing the number of cut points. 

With any $m$ sampled cut points $\{c_i\}_{i=0}^{m-1}$ satisfying that $0<c_0<c_1<\cdots<c_{m-1}<T$, the condition \eqref{strd} for the sub-trajectory reward function can be extended as
\begin{equation}
    \hat{R}_{\text{sub}(\tau_{0:c_0})}+\sum_{i=1}^{m-1}\hat{R}_{\text{sub}(\tau_{c_{i-1}:c_i})}+\hat{R}_{\text{sub}(\tau_{c_{m-1}:T})}\approx R_{\text{ep}}(\tau_{0:T}),
\end{equation}
while $\hat{R}_{\text{sub}}(\varnothing)$ and $\hat{R}_{\text{sub}}(\tau_{0:T}) \approx R_{\text{ep}}(\tau_{0:T})$ none the less hold. If $m=0$, this formulation will just learn a function to fit $R_{\text{ep}}(\tau_{0:T})$, without any attribution of the episodic reward. If $m=T-1$, this formulation will equal step-wise return decomposition \eqref{rd}, which is incapable of bringing in any temporal structural representation. Any $m$ in $\{1,2,\ldots,T-2\}$ can achieve a trade-off between representation and attribution of the episodic reward function.

\begin{figure}[b!]
    \centering
    \subfigure{
        \centering
        \includegraphics[width=0.2\linewidth]{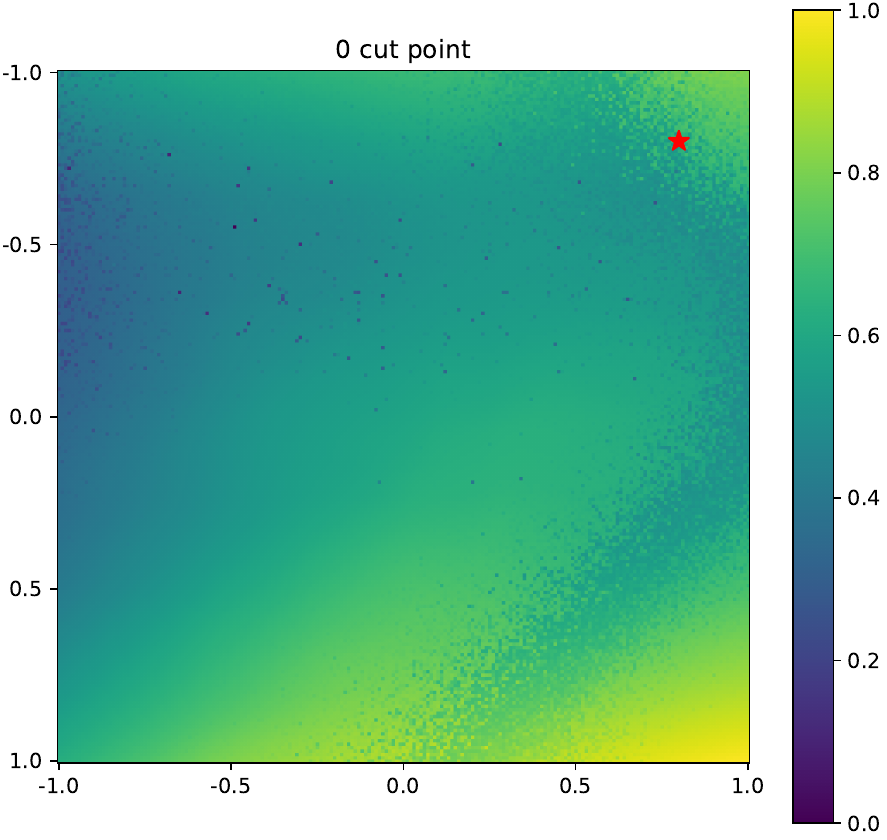}
        \label{cut0}
    }%
    \subfigure{
        \centering
        \includegraphics[width=0.2\linewidth]{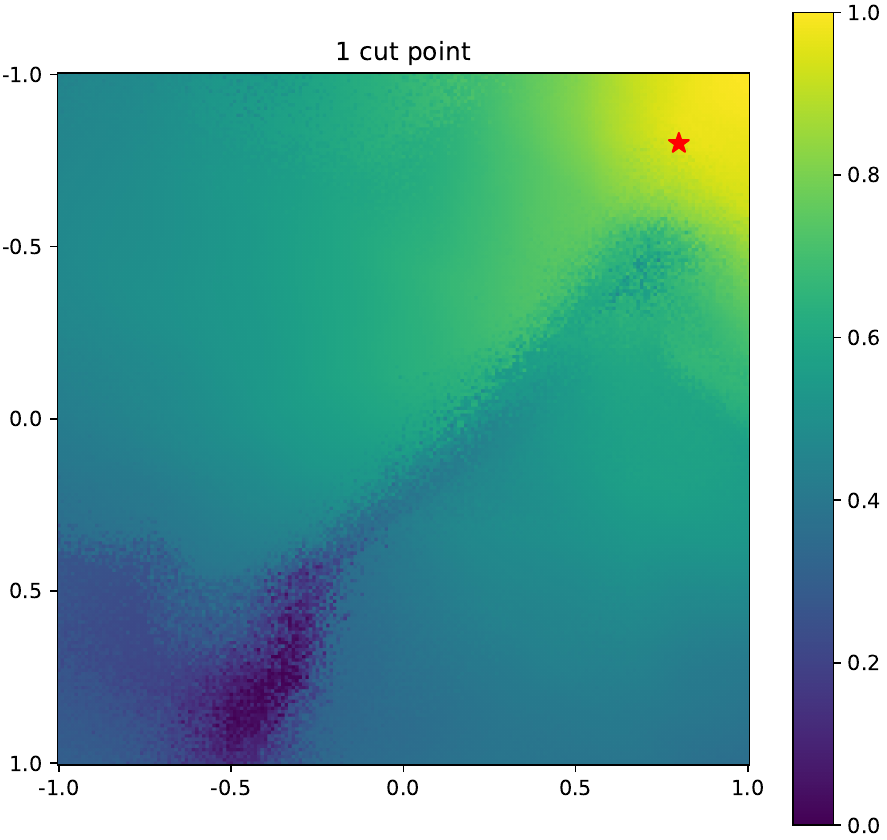}
        \label{cut1}
    } \\
    \subfigure{
        \centering
        \includegraphics[width=0.2\linewidth]{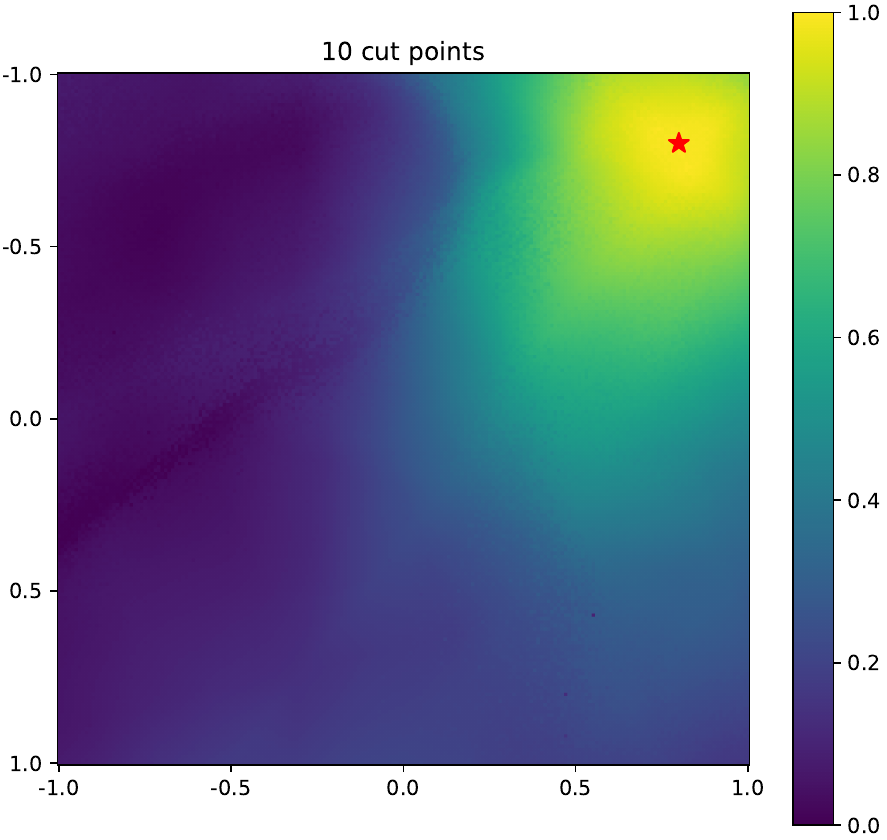}
        \label{cut10}
    }%
    \subfigure{
        \centering
        \includegraphics[width=0.2\linewidth]{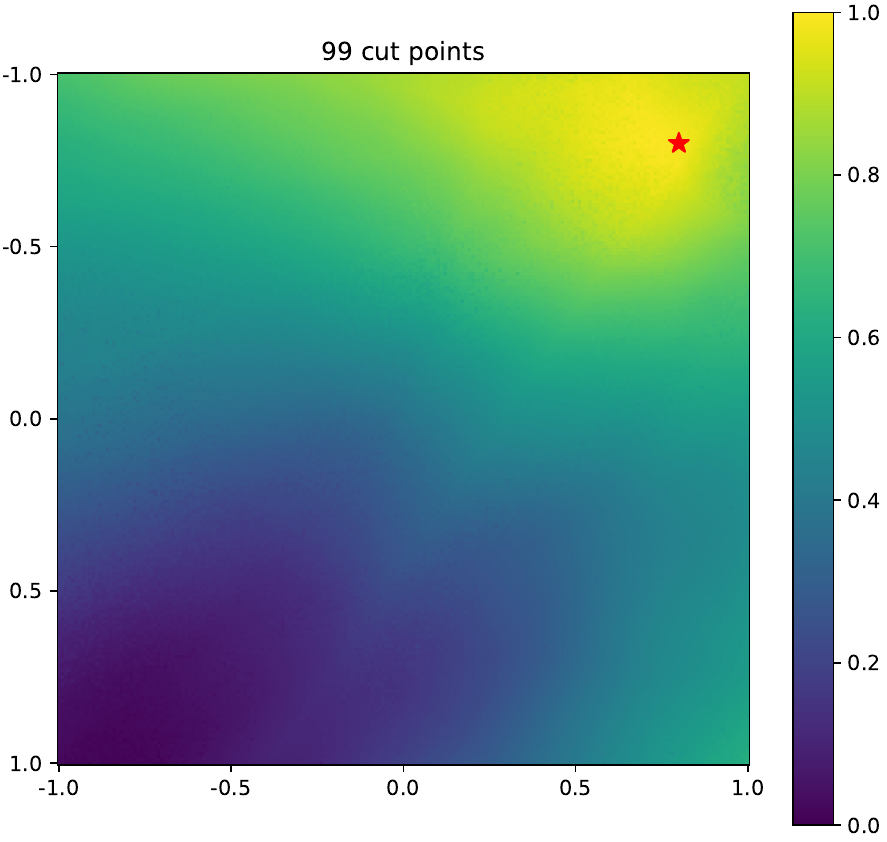}
        \label{cut99}
    }%
    \caption{Heatmap of expected cumulative proxy rewards.}
    \label{heatm}
\end{figure}

First, we visualize the influence of $m$ on the learned proxy value function. For the sake of explainability, the experiment is conducted on one PointMaze task (PointMaze\_UMaze), where a 2-Degree-of-Freedom ball force-actuated in the cartesian directions x and y is aiming to reach a target goal in a closed maze. We plot the heatmap of the normalized expected cumulative proxy rewards in the future for different numbers of cut points, as shown in Figure \ref{heatm}. The red star is the goal on the map. The closer the position (x, y) is to the goal, the greater its expected proxy value tends to be. The tendency of 10 cut points is the clearest, indicating that a proper $m$ between $0$ and $T-1$ can yield reasonable proxy value estimation since the attribution and representation of the episodic return are both considered.

\begin{figure*}[t!]
    \centering
    \includegraphics[width=0.8\linewidth]{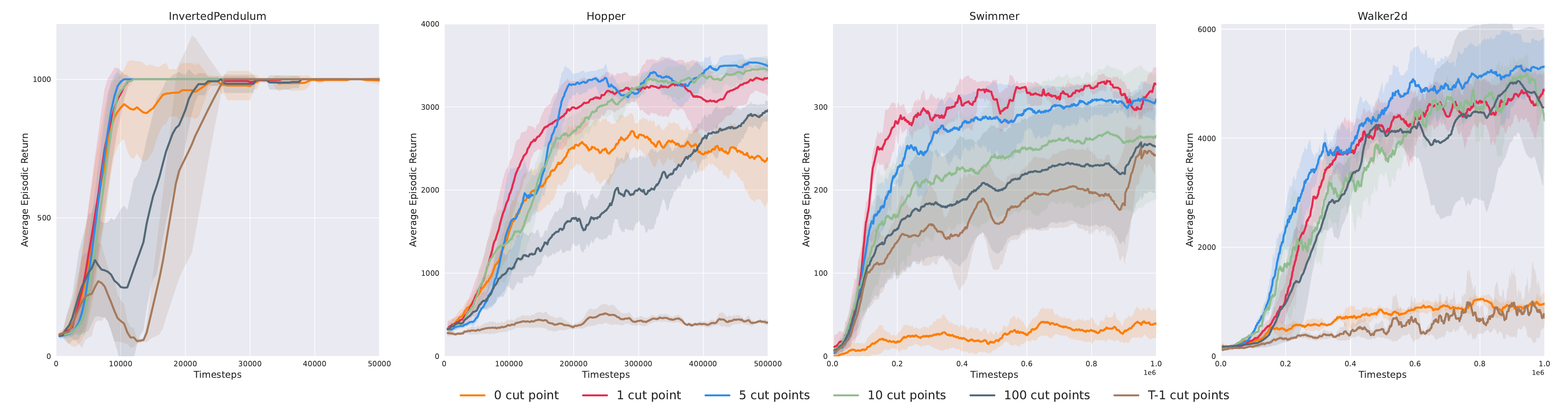}
    \caption{Learning curves of our Diaster with a set of choices of the cut points' number, $m\in\{0,1,5,10,100,T-1\}$, on MuJoCo continuous control tasks. The solid lines indicate the mean and shaded areas indicate the standard error over five different random seeds. Each evaluation, taken every 1,000 environmental steps, calculates the average return over ten episodes.}
    \label{ablation}
\end{figure*}

Next, we evaluate Diaster with a set of choices of the cut points' number, $m\in\{0,1,5,10,100,T-1\}$, on four MuJoCo tasks with the episodic-reward setting, including InvertedPendulum, Hopper, Swimmer, and Walker2d. The results are demonstrated in Figure \ref{ablation}. The performance increases first and then decreases as the number of cut points $m$ increases from $0$ to $T-1$. $m=0$ and $m=T-1$ are always unable to achieve an acceptable performance, because $m=0$ restricts the attribution capability while $m=T-1$ restricts the long-term representation. Although the sensitivity of this hyper-parameter depends on the task, an appropriate integer $m$ between $0$ and $T-1$ can better approximate the return decomposition object and achieve better performance since it can balance attribution and representation of the episodic return. In all the experimental tasks, $m=1$ or $m=5$ performs the best, which suggests that a relatively small $m$ is befitting.

In conclusion, attribution and representation are both necessary for long-term episodic return decomposition. The lack of either of them would lead to poor performance. The trade-off between attribution and representation achieved by our Diaster improves the sample efficiency of RL in the tasks with episodic-reward settings.

\begin{figure*}[t!]
    \centering
    \includegraphics[width=0.8\linewidth]{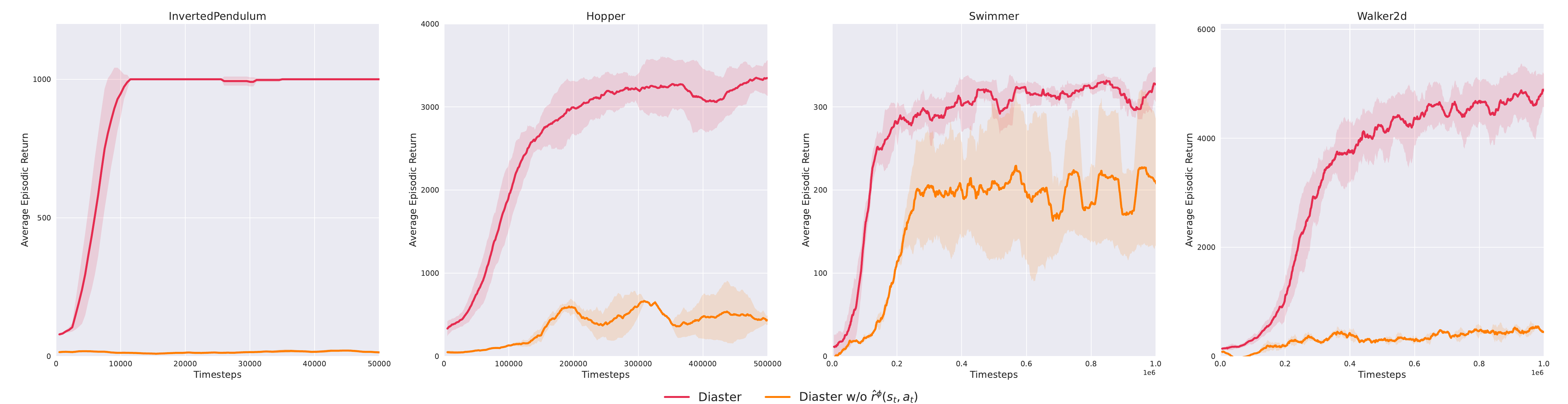}
    \caption{Learning curves of our Diaster, with and without learning the step-wise proxy reward function, respectively.}
    \label{ablation2}
\end{figure*}

\subsection{On Step-wise Proxy Reward Function Learning}
We propose to learn the sub-trajectory reward function $\hat{R}_{\text{sub}}^\psi$ by minimizing \eqref{sub_loss} then learn the step-wise proxy reward function $\hat{r}^\phi$ by minimizing \eqref{step_loss} and use $\hat{r}^\phi$ to provide incentives for the agent. It's also feasible to directly use the difference of $\hat{R}_{\text{sub}}^\psi$ to optimize the policy for each trajectory. In this subsection, we conduct an ablation study to show how Diaster performs without learning $\hat{r}^\phi$.

We evaluate Diaster, with and without learning the step-wise proxy reward function $\hat{r}^\phi$, respectively, on four MuJoCo tasks with the episodic-reward setting, including InvertedPendulum, Hopper, Swimmer, and Walker2d. The results are demonstrated in Figure \ref{ablation2}. We point out that $\hat{R}_{\text{sub}}^\psi(\tau_{0:t})$ is conditioned on the sub-trajectory $\tau_{0:t}$ and the difference of $\hat{R}_{\text{sub}}^\psi$ is non-Markovian. Without a Markovian reward signal, the policy optimization becomes non-stationary. Therefore, the performance deteriorates after removing the step-wise rewards in all four tasks.

\section{Conclusion}
\label{conclu}
In this paper, we present a new episodic return decomposition approach for episodic-reward setting where the only accessible feedback is the episodic reward obtained at the end of an episode. Specifically, we propose to cut an entire trajectory at any time step and decompose the episodic return into two sub-trajectory rewards. Then the step-wise reward can be straightforwardly obtained by differencing the implicitly assigned sub-trajectory reward in expectation. The new method called Diaster not only introduces temporal structural representations into the episodic reward function but also properly attributes the return to different parts of a given trajectory. The theoretical results verify that the step-wise proxy reward function learned through Diaster is return-equivalent to the original MDP in expectation and capable of guiding the policy to be nearly optimal. The empirical results show that Diaster can provide effective proxy rewards for RL algorithms and outperform previous state-of-the-art return decomposition methods in terms of both sample efficiency and performance.

\section{Acknowledgments}
This work is supported by the National Science Foundation of China (61921006), and The Major Key Project of PCL (PCL2021A12).


\bibliography{aaai24}
\bibliographystyle{unsrt}

\clearpage
\appendix
\section{A. Proofs of the Theoretical Results}
\subsection{A.1. Proof of Theorem \ref{theorem1}}
\label{theorem1_proof}
Any environmental state $s_t$ is sampled according to the on-policy distribution $\rho^\pi(s_t)$ and the action $a_t$ is chosen by the policy $\pi(a_t|s_t)$. We unfold the expectation-form as
\begin{equation}
\mathbb{E}_{\rho^\pi}\left[\sum_{i=t}^{t+h-1}\hat{r}^\pi(s_i,a_i)\right]
=\sum_{i=t}^{t+h-1}\sum_{s_i,a_i}\rho^\pi(s_i)\pi(a_i|s_i)\hat{r}^\pi(s_i,a_i).
\end{equation}
Replacing the step-wise proxy reward function $\hat{r}^\pi$ with its definition \eqref{stepr}, we obtain
\begin{equation}
\begin{aligned}
&\mathbb{E}_{\rho^\pi}\left[\sum_{i=t}^{t+h-1}\hat{r}^\pi(s_i,a_i)\right]\\
=&\sum_{i=t}^{t+h-1}\sum_{s_i,a_i}\rho^\pi(s_i)\pi(a_i|s_i)\sum_{\tilde{\tau}_{0:i}}\Gamma^\pi_i(\tilde{\tau}_{0:i}|s_i)\hat{R}_d(\tilde{\tau}_{0:i},s_i,a_i)\\
=&\sum_{i=t}^{t+h-1}\sum_{s_i,a_i}\rho^\pi(s_i)\pi(a_i|s_i)\sum_{\tilde{\tau}_{0:i}}\Gamma^\pi_i(\tilde{\tau}_{0:i}|s_i)\hat{R}_{\text{sub}}(\tilde{\tau}_{0:i},s_i,a_i)\\
&\quad\quad-\sum_{i=t}^{t+h-1}\sum_{s_i,a_i}\rho^\pi(s_i)\pi(a_i|s_i)\sum_{\tilde{\tau}_{0:i}}\Gamma^\pi_i(\tilde{\tau}_{0:i}|s_i)\hat{R}_{\text{sub}}(\tilde{\tau}_{0:i}).
\end{aligned}
\end{equation}
The first term can be rewritten as
\begin{equation}
\begin{aligned}
&\sum_{i=t}^{t+h-1}\sum_{s_i,a_i}\rho^\pi(s_i)\pi(a_i|s_i)\sum_{\tilde{\tau}_{0:i}}\Gamma^\pi_i(\tilde{\tau}_{0:i}|s_i)\hat{R}_{\text{sub}}(\tilde{\tau}_{0:i},s_i,a_i)\\
=&\sum_{i=t}^{t+h-1}\sum_{\tilde{\tau}_{0:i}}\sum_{s_i,a_i}\mathcal{T}_i^\pi(\tilde{\tau}_{0:i})P(s_i|\tilde{s}_{i-1},\tilde{a}_{i-1})\pi(a_i|s_i)\hat{R}_{\text{sub}}(\tilde{\tau}_{0:i},s_i,a_i)\\
=&\sum_{i=t}^{t+h-1}\sum_{\tilde{\tau}_{0:i}}\sum_{s_i,a_i}\mathcal{T}_{i+1}^\pi(\tilde{\tau}_{0:i},s_i,a_i)\hat{R}_{\text{sub}}(\tilde{\tau}_{0:i},s_i,a_i)\\
=&\sum_{i=t}^{t+h-1}\sum_{\tilde{\tau}_{0:i+1}}\mathcal{T}_{i+1}^\pi(\tilde{\tau}_{0:i+1})\hat{R}_{\text{sub}}(\tilde{\tau}_{0:i+1}),
\end{aligned}
\end{equation}
while the second term can be rewritten as
\begin{equation}
\begin{aligned}
&\sum_{i=t}^{t+h-1}\sum_{s_i,a_i}\rho^\pi(s_i)\pi(a_i|s_i)\sum_{\tilde{\tau}_{0:i}}\Gamma^\pi_i(\tilde{\tau}_{0:i}|s_i)\hat{R}_{\text{sub}}(\tilde{\tau}_{0:i})\\
=&\sum_{i=t}^{t+h-1}\sum_{s_i,a_i}\rho^\pi(s_i)\pi(a_i|s_i)\sum_{\tilde{\tau}_{0:i}}\frac{\mathcal{T}_i^\pi(\tilde{\tau}_{0:i})P(s_i|\tilde{s}_{i-1},\tilde{a}_{i-1})}{\rho^\pi(s_i)}\hat{R}_{\text{sub}}(\tilde{\tau}_{0:i})\\
=&\sum_{i=t}^{t+h-1}\sum_{\tilde{\tau}_{0:i}}\mathcal{T}_i^\pi(\tilde{\tau}_{0:i})\hat{R}_{\text{sub}}(\tilde{\tau}_{0:i}).
\end{aligned}
\end{equation}
Since the first term is partial-cancellable with the second term, the final result is obtained as
\begin{equation}
\begin{aligned}
&\mathbb{E}_{\rho^\pi}\left[\sum_{i=t}^{t+h-1}\hat{r}^\pi(s_i,a_i)\right]\\
=&\sum_{\tau_{0:t+h}}\mathcal{T}_{t+h}^\pi(\tau_{0:t+h})\hat{R}_{\text{sub}}(\tau_{0:t+h})
-\sum_{\tau_{0:t}}\mathcal{T}_i^\pi(\tau_{0:t})\hat{R}_{\text{sub}}(\tau_{0:t})\\
=&\mathbb{E}_{\rho^\pi}\left[\hat{R}_{\text{sub}}(\tau_{0:t+h})-\hat{R}_{\text{sub}}(\tau_{0:t})\right].
\end{aligned}
\end{equation}

\subsection{A.2. Proof of Lemma \ref{theorem2}}
\label{theorem2_proof}
For value-based methods \cite{dqn, doubleq, duelq}, the policy is improved by directly applying $\arg\max_a Q(s,a)$. Since it is observed that
\begin{equation}
    \arg\max_{a} \hat{Q}^\pi(s,a)=\arg\max_{a} [Q^\pi(s,a)+\delta(s)]=\arg\max_{a} Q^\pi(s,a),
\end{equation}
for any $s\in\mathcal{S}$, any policy $\pi$ can be updated to the same policy during the policy improvement stage at each policy iteration, whether evaluated by $\hat{Q}^\pi$ or $Q^\pi$.

For policy gradient methods \cite{pg, trpo, sac, ppo}, given any policy $\pi_\theta$ parameterized by $\theta$, the derivation
\begin{equation}
\begin{aligned}
&\mathbb{E}_{s_t,a_t}\left[\nabla_\theta\log\pi_\theta(a_t|s_t)\hat{Q}^{\pi_\theta}(s_t,a_t)\right]\\
=&\mathbb{E}_{s_t,a_t}\left[\nabla_\theta\log\pi_\theta(a_t|s_t)\left[Q^{\pi_\theta}(s_t,a_t)+\delta(s_t)\right]\right]\\
=&\mathbb{E}_{s_t,a_t}\left[\nabla_\theta\log\pi_\theta(a_t|s_t)Q^{\pi_\theta}(s_t,a_t)\right]
+\mathbb{E}_{s_t,a_t}\left[\nabla_\theta\log\pi_\theta(a_t|s_t)\delta(s_t)\right]\\
=&\mathbb{E}_{s_t,a_t}\left[\nabla_\theta\log\pi_\theta(a_t|s_t)Q^{\pi_\theta}(s_t,a_t)\right]
+\mathbb{E}_{s_t}\left[\sum_{a_t}\pi_\theta(a_t|s_t)\nabla_\theta\log\pi_\theta(a_t|s_t)\delta(s_t)\right]\\
=&\mathbb{E}_{s_t,a_t}\left[\nabla_\theta\log\pi_\theta(a_t|s_t)Q^{\pi_\theta}(s_t,a_t)\right]
+\mathbb{E}_{s_t}\left[\delta(s_t)\sum_{a_t}\nabla_\theta\pi_\theta(a_t|s_t)\right]\\
=&\mathbb{E}_{s_t,a_t}\left[\nabla_\theta\log\pi_\theta(a_t|s_t)Q^{\pi_\theta}(s_t,a_t)\right]
+\mathbb{E}_{s_t}\left[\delta(s_t)\nabla_\theta\sum_{a_t}\pi_\theta(a_t|s_t)\right]\\
=&\mathbb{E}_{s_t,a_t}\left[\nabla_\theta\log\pi_\theta(a_t|s_t)Q^{\pi_\theta}(s_t,a_t)\right]\\
\end{aligned}
\end{equation}
indicates that the policy gradient computed by $\hat{Q}^\pi$ is the same as $Q^\pi$ in expectation, thus leading to the same optimal policy.

\subsection{A.3. Proof of Theorem \ref{theorem3}}
\label{theorem3_proof}
Any sub-trajectory starting from the pair $(s_t, a_t)$ is sampled according to the dynamics function $P$ and the policy $\pi$. The probability of the state $s_{t+h}$ ($h\in\{1,\ldots,T-t-1\}$) conditioned on $(s_t,a_t)$ is given by
\begin{equation}
\mu^\pi(s_{t+h}|s_t,a_t)=\sum_{s_{t+1},a_{t+1}}\sum_{s_{t+2},a_{t+2}}\cdots\sum_{s_{t+h-1},a_{t+h-1}}P(s_{t+1}|s_{t},a_{t})\left(\prod_{i=1}^{h-1}\pi(a_{t+i}|s_{t+i})P(s_{t+i+1}|s_{t+i},a_{t+i})\right),
\end{equation}
thus we can unfold the expectation-form as
\begin{equation}
\mathbb{E}_{P,\pi}\left[\left.\sum_{i=1}^{T-t-1}\hat{r}^\pi(s_{t+i},a_{t+i})\right|s_t,a_t\right]
=\sum_{i=1}^{T-t-1}\sum_{s_{t+i},a_{t+i}}\mu^\pi(s_{t+i}|s_t,a_t)\pi(a_{t+i}|s_{t+i})\hat{r}^\pi(s_{t+i},a_{t+i})
\end{equation}
Replacing the step-wise proxy reward function $\hat{r}^\pi$ with its definition \eqref{stepr}, we obtain
\begin{equation}
\begin{aligned}
&\mathbb{E}_{P,\pi}\left[\left.\sum_{i=1}^{T-t-1}\hat{r}^\pi(s_{t+i},a_{t+i})\right|s_t,a_t\right]\\
=&\sum_{i=1}^{T-t-1}\sum_{s_{t+i},a_{t+i}}\mu^\pi(s_{t+i}|s_t,a_t)\pi(a_{t+i}|s_{t+i})\sum_{\tilde{\tau}_{0:t+i}}\Gamma^\pi_{t+i}(\tilde{\tau}_{0:{t+i}}|s_{t+i})\hat{R}_d(\tilde{\tau}_{0:t+i},s_{t+i},a_{t+i})\\
=&\sum_{i=1}^{T-t-1}\sum_{s_{t+i},a_{t+i}}\mu^\pi(s_{t+i}|s_t,a_t)\pi(a_{t+i}|s_{t+i})\sum_{\tilde{\tau}_{0:t+i}}\Gamma^\pi_{t+i}(\tilde{\tau}_{0:t+i}|s_{t+i})\hat{R}_\text{sub}(\tilde{\tau}_{0:t+i},s_{t+i},a_{t+i})\\
&\quad\quad-\sum_{i=1}^{T-t-1}\sum_{s_{t+i},a_{t+i}}\mu^\pi(s_{t+i}|s_t,a_t)\pi(a_{t+i}|s_{t+i})\sum_{\tilde{\tau}_{0:t+i}}\Gamma^\pi_{t+i}(\tilde{\tau}_{0:t+i}|s_{t+i})\hat{R}_\text{sub}(\tilde{\tau}_{0:t+i}).
\end{aligned}
\end{equation}
Letting $\Upsilon^\pi_{t+i+1}(\tilde{\tau}_{0:t+i},s_{t+i},a_{t+i}|s_t,a_t)=\mu^\pi(s_{t+i}|s_t,a_t)\pi(a_{t+i}|s_{t+i})\Gamma^\pi_{t+i}(\tilde{\tau}_{0:t+i}|s_{t+i})$, the first term can be rewritten as
\begin{equation}
\begin{aligned}
&\sum_{i=1}^{T-t-1}\sum_{s_{t+i},a_{t+i}}\mu^\pi(s_{t+i}|s_t,a_t)\pi(a_{t+i}|s_{t+i})\sum_{\tilde{\tau}_{0:t+i}}\Gamma^\pi_{t+i}(\tilde{\tau}_{0:t+i}|s_{t+i})\hat{R}_\text{sub}(\tilde{\tau}_{0:t+i},s_{t+i},a_{t+i})\\
=&\sum_{i=1}^{T-t-1}\sum_{s_{t+i},a_{t+i}}\sum_{\tilde{\tau}_{0:t+i}}\Upsilon^\pi_{t+i+1}(\tilde{\tau}_{0:t+i},s_{t+i},a_{t+i}|s_t,a_t)\hat{R}_\text{sub}(\tilde{\tau}_{0:t+i},s_{t+i},a_{t+i})\\
=&\sum_{i=1}^{T-t-1}\sum_{\tilde{\tau}_{0:t+i+1}}\Upsilon^\pi_{t+i+1}(\tilde{\tau}_{0:t+i+1}|s_t,a_t)\hat{R}_\text{sub}(\tilde{\tau}_{0:t+i+1}),
\end{aligned}
\end{equation}
while the second term can be rewritten as
\begin{equation}
\begin{aligned}
&\sum_{i=1}^{T-t-1}\sum_{s_{t+i},a_{t+i}}\mu^\pi(s_{t+i}|s_t,a_t)\pi(a_{t+i}|s_{t+i})\sum_{\tilde{\tau}_{0:t+i}}\Gamma^\pi_{t+i}(\tilde{\tau}_{0:t+i}|s_{t+i})\hat{R}_\text{sub}(\tilde{\tau}_{0:t+i})\\
=&\sum_{i=1}^{T-t-1}\sum_{s_{t+i},a_{t+i}}\sum_{\tilde{\tau}_{0:t+i}}\Upsilon^\pi_{t+i+1}(\tilde{\tau}_{0:t+i},s_{t+i},a_{t+i}|s_t,a_t)\hat{R}_\text{sub}(\tilde{\tau}_{0:t+i})\\
=&\sum_{i=1}^{T-t-1}\sum_{\tilde{\tau}_{0:t+i}}\Upsilon^\pi_{t+i+1}(\tilde{\tau}_{0:t+i}|s_t,a_t)\hat{R}_\text{sub}(\tilde{\tau}_{0:t+i}).
\end{aligned}
\end{equation}
Since the first term is partial-cancellable with the second term, the final result is obtained as
\begin{equation}
\begin{aligned}
&\mathbb{E}_{P,\pi}\left[\left.\sum_{i=1}^{T-t-1}\hat{r}^\pi(s_{t+i},a_{t+i})\right|s_t,a_t\right]\\
=&\sum_{\tau_{0:T}}\Upsilon^\pi_{T}(\tau_{0:T}|s_t,a_t)\hat{R}_\text{sub}(\tau_{0:T})
-\sum_{\tau_{0:t+1}}\Upsilon^\pi_{t+1}(\tau_{0:t+1}|s_t,a_t)\hat{R}_\text{sub}(\tau_{0:t+1})\\
=&\mathbb{E}_{P,\pi}\left[\left.\hat{R}_\text{sub}(\tau_{0:T})-\hat{R}_\text{sub}(\tau_{0:t+1})\right|s_t,a_t\right].
\end{aligned}
\end{equation}

\section{B. Hyper-parameters}
\label{hyper}

\begin{table*}[h!]
\centering
\begin{tabular}{cc}
\toprule
Hyper-parameter     &  Default\\
\midrule
 hidden layers (all networks)  & 2\\
 hidden neurons & 256 \\
 activation & ReLU \cite{relu} \\
 optimizer (all losses) & Adam \cite{adam}\\
\midrule
RL-ALGO & SAC \cite{sac} \\
learning rate (all networks) & $3\times10^{-4}$\\
discount factor & 0.99 \\
initial temperature & 1.0\\
target entropy & -dim($\mathcal{A}$)\\
target network smoothing coefficient & 0.005\\
\midrule
replay buffer capacity & $10^6$\\
mini-batch size & 256 \\
\bottomrule
\end{tabular}
\caption{The hyper-parameter configuration of Diaster in MuJoCo experiments.}
\end{table*}

\end{document}